\documentclass[12pt]{article}
\usepackage[centertags]{amsmath}
\usepackage{amsfonts}
\usepackage{amssymb}
\usepackage{latexsym}
\usepackage{amsthm}
\usepackage{newlfont}
\usepackage{graphicx}
\usepackage{listings}
\usepackage{booktabs}
\usepackage{abstract}
\newlength{\defbaselineskip}
\newcommand{\setlinespacing}[1]%
           {\setlength{\baselineskip}{#1 \defbaselineskip}}

\newcommand{\actaqed}{\hfill $\actabox$}
{\medskip\noindent \textit{Proof of #1. }}%
{\actaqed \medskip}

\def\D{{\mathcal D}}

\def\C{{\mathcal C}}
\def\R{{\mathbb R}}
\def \<{\langle}
\def\>{\rangle}

\def \e{\epsilon}

\def \ff{\varphi}

\def \sp{\operatorname{span}}

\def\a{\alpha}
\def\ga{\gamma}

\newtheorem{Theorem}{Theorem}[section]
\newtheorem{Lemma}{Lemma}[section]

\numberwithin{equation}{section}

\begin{document}
\title{{Greedy expansions in convex optimization} }
\author{V.N. Temlyakov \thanks{ University of South Carolina. Research was supported by NSF grant DMS-0906260 }} \maketitle
\begin{abstract}
{This paper is a follow up to the previous author's paper on convex optimization. In that paper we began the process of adjusting greedy-type algorithms from nonlinear approximation for finding sparse solutions of convex optimization problems. We modified there three the most popular in nonlinear approximation in Banach spaces greedy algorithms -- Weak Chebyshev Greedy Algorithm, Weak Greedy Algorithm with Free Relaxation and Weak Relaxed Greedy Algorithm -- for solving convex optimization problems. 
We continue to study sparse approximate solutions to convex optimization problems. It is known that in many engineering applications researchers are interested in 
an approximate solution of an optimization problem as a linear combination of elements from a given system of elements. There is an increasing interest in building such sparse approximate solutions using different greedy-type algorithms. In this paper we concentrate on greedy algorithms that provide expansions, which means that the approximant at the $m$th iteration is equal to the sum of the approximant from the previous iteration ($(m-1)$th iteration) and one element from the dictionary with an appropriate coefficient. The problem of greedy expansions of elements of a Banach space is well studied 
in nonlinear approximation theory.  At a first glance the setting of a problem of
expansion of a given element and the setting of the problem of expansion in an optimization problem are very different. However, it turns out that the same technique can be used for solving both problems.
We show how the technique developed in nonlinear approximation theory,
in particular, the greedy expansions technique can be adjusted for finding a sparse solution  of an optimization problem  given by an expansion with respect to a given dictionary.    }
\end{abstract}

\section{Introduction}

This paper is a follow up to the author's paper \cite{T} on convex optimization. In \cite{T} we began the process of adjusting greedy-type algorithms from nonlinear approximation for finding sparse solutions of convex optimization problems. We modified in \cite{T} three the most popular in nonlinear approximation in Banach spaces greedy algorithms -- Weak Chebyshev Greedy Algorithm, Weak Greedy Algorithm with Free Relaxation and Weak Relaxed Greedy Algorithm -- for solving convex optimization problems. 
We continue to study sparse approximate solutions to convex optimization problems. We apply the technique developed in nonlinear approximation known under the name of {\it greedy approximation}. A typical 
problem of convex optimization is to find an approximate solution to the problem
\begin{equation}\label{1.0}
\inf_x E(x)
\end{equation}
under assumption that $E$ is a convex function. Usually, in convex optimization function $E$ is defined on a finite dimensional space $\R^n$ (see \cite{BL}, \cite{N}).
Recent needs of numerical analysis call for consideration of the above optimization problem on an infinite dimensional space, for instance, a space of 
continuous functions. Thus, we consider a convex function $E$ defined on a Banach space $X$. It is pointed out in \cite{Z} that in many engineering applications researchers are interested in 
an approximate solution of problem (\ref{1.0}) as a linear combination of elements from a given system $\D$ of elements. There is an increasing interest in building such sparse approximate solutions using different greedy-type algorithms (see, for instance, \cite{Z}, \cite{SSZ}, \cite{CRPW},\cite{TRD}, and \cite{T}). The problem of approximation of a given element $f\in X$ by linear combinations of elements from $\D$ is well studied 
in nonlinear approximation theory (see, for instance \cite{D}, \cite{T2}, \cite{Tbook}). Many of known greedy-type algorithms provide such approximation in a form of expansion of a given element into a series with respect to a given dictionary $\D$.  In the paper \cite{T} we showed how some of the greedy algorithms that provide good approximation, but not an expansion, can be adjusted for solving an optimization problem. In this paper we concentrate on greedy algorithms that provide expansions, which means that the approximant at the $m$th iteration is equal to the sum of the approximant from the previous iteration ($(m-1)$th iteration) and one element from the dictionary with an appropriate coefficient.

  We point out that at a first glance the setting of a problem of
expansion of a given element and the setting of the expansion problem in an optimization  are very different. However, it turns out that the same technique can be used for solving both problems.
We show how the technique developed in nonlinear approximation theory,
in particular, the greedy expansions technique can be adjusted for finding a sparse solution  of an optimization problem (\ref{1.0}) given by an expansion with respect to a given dictionary $\D$. 
 
We begin with a brief description of greedy expansion methods in Banach spaces. Let $X$ be a Banach space with norm $\|\cdot\|$. We say that a set of elements (functions) $\D$ from $X$ is a dictionary, respectively, symmetric dictionary, if each $g\in \D$ has norm bounded by one ($\|g\|\le1$),
$$
g\in \D \quad \text{implies} \quad -g \in \D,
$$
and the closure of $\sp \D$ is $X$. In this paper symmetric dictionaries are considered. We denote the closure (in $X$) of the convex hull of $\D$ by $A_1(\D)$.   For a nonzero element $f\in X$ we let $F_f$ denote a norming (peak) functional for $f$: 
$$
\|F_f\| =1,\qquad F_f(f) =\|f\|.
$$
The existence of such a functional is guaranteed by Hahn-Banach theorem.

We 
assume that the set
$$
D:=\{x:E(x)\le E(0)\}
$$
is bounded.
For a bounded set $S$ define the modulus of smoothness of $E$ on $S$ as follows
\begin{equation}\label{1.1}
\rho(E,u):=\rho(E,S,u):=\frac{1}{2}\sup_{x\in S, \|y\|=1}|E(x+uy)+E(x-uy)-2E(x)|.
\end{equation}

We assume that $E$ is Fr{\'e}chet differentiable. Then convexity of $E$ implies that for any $x,y$ 
\begin{equation}\label{1.2}
E(y)\ge E(x)+\<E'(x),y-x\>
\end{equation}
or, in other words,
\begin{equation}\label{1.3}
E(x)-E(y) \le \<E'(x),x-y\> = \<-E'(x),y-x\>.
\end{equation} 
We will often use the following simple lemma.
\begin{Lemma}\label{L1.1} Let $E$ be Fr{\'e}chet differentiable convex function. Then the following inequality holds for $x\in S$
\begin{equation}\label{1.4}
0\le E(x+uy)-E(x)-u\<E'(x),y\>\le 2\rho(E,S,u\|y\|).  
\end{equation}
\end{Lemma}
\begin{proof} The left inequality follows directly from (\ref{1.2}).
  Next, from the definition of modulus of smoothness it follows that
\begin{equation}\label{1.5}
E(x+uy)+E(x-uy)\le 2(E(x)+\rho(E,S,u\|y\|)).  
\end{equation}
Inequality (\ref{1.2}) gives
\begin{equation}\label{1.6}
E(x-uy)\ge E(x) + \<E'(x),-uy\>=E(x)-u\<E'(x),y\>. 
\end{equation}
Combining (\ref{1.5}) and (\ref{1.6}), we obtain
$$
E(x+uy)\le E(x)+u\<E'(x),y\>+2\rho(E,S,u\|y\|).
$$
This proves the second inequality. 
\end{proof}

From the definition of a dictionary it follows that any element $f\in X$ can be approximated arbitrarily well by finite linear combinations of the dictionary elements. The primary goal of greedy expansion theory is to study representations of an element $f\in X$ by a series
\begin{equation}\label{1.7}
f\sim \sum_{j=1}^\infty c_j(f)g_j(f), \quad g_j(f) \in \D,\quad c_j(f)>0, \quad j=1,2,\dots.  
\end{equation}
In building the representation (\ref{1.7}) we should construct two sequences: 
\newline $\{g_j(f)\}_{j=1}^\infty$ and $\{c_j(f)\}_{j=1}^\infty$.  In greedy expansion theory the construction of $\{g_j(f)\}_{j=1}^\infty$ is based on ideas used in greedy-type nonlinear approximation (greedy-type algorithms). This justifies the use of the term {\it greedy expansion} for (\ref{1.7}).  The construction of $\{g_j(f)\}_{j=1}^\infty$ is, clearly, the most important and difficult part in building the representation (\ref{1.7}). On the basis of the contemporary theory of nonlinear approximation with respect to redundant dictionaries, we may conclude that the method of using a norming functional in greedy steps of an algorithm is the most productive in approximation in Banach spaces.  

Denote 
$$
r_\D(f) := \sup_{F_f}\|F_f\|_\D:=\sup_{F_f}\sup_{g\in \D}F_f(g).
$$
We note that, in general, a norming functional $F_f$ is not unique. This is why we take $\sup_{F_f}$ over all norming functionals of $f$ in the definition of $r_\D(f)$. It is known that in the case of uniformly smooth Banach spaces (our primary object here) the norming functional $F_f$ is unique. In such a case we do not need $\sup_{F_f}$ in the definition of $r_\D(f)$, we have $r_\D(f)=\|F_f\|_\D$.
 
We begin with a description of a general scheme that provides an expansion for a given element $f$. Later, specifying this general scheme, we will obtain different methods of expansion. 

 {\bf Dual-Based Expansion (DBE).} Let $t\in (0,1]$ and $f\neq 0$. Denote $f_0:=f$. Assume $\{f_j\}_{j=0}^{m-1} \subset X$, $\{\ff_j\}_{j=1}^{m-1} \subset \D$ and a set of coefficients $\{c_j\}_{j=1}^{m-1}$ of expansion have already been constructed. If $f_{m-1}=0$ then we stop (set $c_j=0$, $j=m,m+1,\dots$ in the expansion) and get $f=\sum_{j=1}^{m-1}c_j\ff_j$. If $f_{m-1} \neq 0$ then we conduct the following two steps.

(1) Choose $\ff_m\in\D$ such that 
$$
\sup_{F_{f_{m-1}}}F_{f_{m-1}}(\ff_m) \ge tr_\D(f_{m-1}).
$$

(2) Define
$$
f_m:= f_{m-1}-c_m\ff_m,
$$
where $c_m>0$ is a coefficient either prescribed in advance or chosen from a concrete approximation procedure. 

We call the series
\begin{equation}\label{1.8}
f \sim \sum_{j=1}^\infty c_j\ff_j  
\end{equation}
the Dual-Based Expansion of $f$ with coefficients $c_j(f):=c_j$, $j=1,2,\dots$ with respect to $\D$.
 
Denote
$$
S_m(f,\D) := \sum_{j=1}^m c_j\ff_j.
$$
Then it is clear that 
$$
f_m = f-S_m(f,\D).
$$
The reader can find some convergence results for the DBE in Sections 6.7.2 and 6.7.3 of \cite{Tbook}.
 
Let $\C:=\{c_m\}_{m=1}^\infty$ be a fixed sequence of positive numbers. We restrict ourselves to positive numbers because of the symmetry of the dictionary $\D$.  

 {\bf $X$-Greedy Algorithm with coefficients $\C$ (XGA($\C$)).} We define $f_0:=f$, $G_0:=0$. Then, for each $m\ge 1$ we have the following inductive definition.

(1) $\ff_m\in\D$ is such that (assuming existence)
$$
 \|f_{m-1}-c_m\ff_m\|_X=\inf_{g\in\D}\|f_{m-1}-c_m g\|_X. 
$$

(2) Let
$$
f_m:=f_{m-1}-c_m\ff_m,\qquad G_m:=G_{m-1}+c_m\ff_m.
$$

 {\bf Dual Greedy Algorithm with weakness $\tau$ and coefficients $\C$\newline (DGA($\tau,\C$)).} Let $\tau:=\{t_m\}_{m=1}^\infty$, $t_m\in[0,1]$, be a weakness sequence. 
We define $f_0 :=  f$, $G_0:=0$. Then, for each $m\ge 1$ we have the following inductive definition.

(1) $\varphi_m  \in \D$ is any element satisfying
$$
F_{f_{m-1}}(\varphi_m) \ge t_m  \|F_{f_{m-1}}\|_\D. 
$$

(2) Let  
$$
f_m :=   f_{m-1}-c_m\varphi_m,\qquad G_m:=G_{m-1}+c_m\ff_m.
$$
 
In   the case $\tau=\{t\}$, $t\in(0,1]$, we write $t$ instead of $\tau$ in the notation.

It is easy to see that for any Banach space $X$ its modulus of smoothness $\rho(u)$ is an even convex function satisfying the inequalities
$$
\max(0,u-1)\le \rho(u)\le u,\quad u\in (0,\infty).
$$

In Section 6.7.3 of \cite{Tbook} we considered a variant of the Dual-Based Expansion with coefficients chosen by a certain simple rule. The rule depends on two numerical 
parameters, $t\in (0,1]$ (the weakness parameter from the definition of the DBE) and $b\in (0,1)$ (the tuning parameter of the approximation method). The rule also depends on a majorant $\mu$ of the modulus of smoothness of the Banach space $X$. 

Let $X$ be a uniformly smooth Banach space with modulus of smoothness $\rho(u)$, and let $\mu(u)$ be a continuous majorant of $\rho(u)$: $\rho(u)\le\mu(u)$, $u\in[0,\infty)$ such that $\mu(u)/u$ goes to $0$ monotonically. It is clear that $\mu(2)\ge 1$.

 {\bf Dual Greedy Algorithm with parameters $(t,b,\mu)$ (DGA$(t,b,\mu)$).}
Let $X$ and $\mu(u)$ be as above.   For parameters $t\in(0,1]$, $b\in (0,1]$ we define sequences
$\{f_m\}_{m=0}^\infty$, $\{\ff_m\}_{m=1}^\infty$, $\{c_m\}_{m=1}^\infty$ inductively. Let $f_0:=f$. If for $m\ge 1$ $f_{m-1}=0$ then we set $f_j=0$ for $j\ge m$ and stop. If $f_{m-1}\neq 0$ then we conduct the following three steps.

(1) Take any $\ff_m \in \D$ such that
$$
F_{f_{m-1}}(\ff_m) \ge tr_\D(f_{m-1}).  
$$

(2) Choose $c_m>0$ from the equation
$$
\|f_{m-1}\|\mu(c_m/\|f_{m-1}\|) = \frac{tb}{2}c_mr_\D(f_{m-1}).  
$$

(3) Define
$$
f_m:=f_{m-1}-c_m\ff_m.  
$$

We note that (2) is equivalent to solving the equation
$$
 \frac{\mu(c_m/\|f_{m-1}\|)}{c_m/\|f_{m-1}\|} = \frac{tb}{2}r_\D(f_{m-1}).  
$$
 It follows from the definitions of $t$, $b$ and $r_\D(f_{m-1})$ that the right hand side of the above equation is $\le 1/2$. Therefore, there always exists a unique solution to this equation and it satisfies the inequality
 $$
 c_m/\|f_{m-1}\|\le 2.
 $$
 
For illustration we present two theorems on convergence and rate of convergence of the DGA($\tau,b,\mu$) 
(see Section 6.7.3 of \cite{Tbook}).  

\begin{Theorem}\label{T1.1} Let $X$ be a uniformly smooth Banach space with the modulus of smoothness $\rho(u)$ and let $\mu(u)$ be a continuous majorant of $\rho(u)$ with the property $\mu(u)/u \downarrow 0$ as $u\to +0$. Then, for any $t\in (0,1]$ and $b\in (0,1)$ the DGA$(t,b,\mu)$ converges for each dictionary $\D$ and all $f\in X$.
\end{Theorem}

\begin{Theorem}\label{T1.2} Assume $X$ has a modulus of smoothness $\rho(u)\le \gamma u^q$, $q\in (1,2]$ and $b\in(0,1)$. Denote $\mu(u) = \gamma u^q$. Then, for any dictionary $\D$ and any $f\in A_1(\D)$, the rate of convergence of the DGA$(t,b,\mu)$ is given by 
$$
\|f_m\|\le C(t,b,\gamma,q)m^{-\frac{t(1-b)}{p(1+t(1-b))}}, \quad p:= \frac{q}{q-1}.
$$
\end{Theorem}

We now formulate the corresponding generalizations of the above algorithms to the case of smooth convex function $E$.
Denote 
$$
E_\D(x) :=  \sup_{g\in \D}\<-E'(x),g\>.
$$

{\bf Gradient Based Expansion.} Let $t\in (0,1]$. Denote $G_0:=0$. Assume $\{G_j\}_{j=0}^{m-1} \subset X$, $\{\ff_j\}_{j=1}^{m-1} \subset \D$ and a set of coefficients $\{c_j\}_{j=1}^{m-1}$ of expansion have already been constructed. If $E'(G_{m-1})=0$ then we stop (set $c_j=0$, $j=m,m+1,\dots$ in the expansion). If $E'(G_{m-1}) \neq 0$ then we conduct the following two steps.

(1) Choose $\ff_m\in\D$ such that 
$$
 \<-E'(G_{m-1}),\ff_m\>\ge tE_\D(G_{m-1}).
$$

(2) Define
$$
G_m:= G_{m-1}+c_m\ff_m,
$$
where $c_m>0$ is a coefficient either prescribed in advance or chosen from a concrete approximation procedure. 

We call the series
\begin{equation}\label{1.9}
  \sum_{j=1}^\infty c_j\ff_j 
\end{equation}
the Gradient Based Expansion  with coefficients $ c_j$, $j=1,2,\dots$ with respect to $\D$.

Let $\C:=\{c_m\}_{m=1}^\infty$ be a fixed sequence of positive numbers. We restrict ourselves to positive numbers because of the symmetry of the dictionary $\D$.  

 {\bf $E$-Greedy Algorithm with coefficients $\C$ (EGA($\C$)).} We define  $G_0:=0$. Then, for each $m\ge 1$ we have the following inductive definition.

(1) $\ff_m\in\D$ is such that (assuming existence)
$$
 E(G_{m-1}+c_m\ff_m)=\inf_{g\in\D}E(G_{m-1}+c_m g). 
$$

(2) Let
$$
G_m:=G_{m-1}+c_m\ff_m .
$$

 {\bf Gradient Greedy Algorithm with weakness $\tau$ and coefficients $\C$\newline (GGA($\tau,\C$)).} Let $\tau:=\{t_m\}_{m=1}^\infty$, $t_m\in[0,1]$, be a weakness sequence. 
We define  $G_0:=0$. Then, for each $m\ge 1$ we have the following inductive definition.

(1) $\varphi_m  \in \D$ is any element satisfying
$$
\<-E'(G_{m-1}),\varphi_m\> \ge t_m  E_\D(G_{m-1}). 
$$

(2) Let  
$$
  G_m:=G_{m-1}+c_m\ff_m.
$$
 
In   the case $\tau=\{t\}$, $t\in(0,1]$, we write $t$ instead of $\tau$ in the notation.
 
 Let $E$ be a uniformly smooth convex function with modulus of smoothness $\rho(E,D,u)$, and let $\mu(u)$ be a continuous majorant of $\rho(E,D,u)$: $\rho(E,D,u)\le\mu(u)$, $u\in[0,\infty)$ such that $\mu(u)/u$ goes to $0$ monotonically.  
 
{\bf Gradient Greedy Algorithm with parameters $(\tau,b,\mu)$ (GGA$(\tau,b,\mu)$).}
Let $E$ and $\mu(u)$ be as above. For parameters $\tau=\{t_k\}$, $t_k\in(0,1]$, $b\in (0,1]$ we define sequences
$\{G_m\}_{m=0}^\infty$, $\{\ff_m\}_{m=1}^\infty$, $\{c_m\}_{m=1}^\infty$ inductively. Let $G_0:=0$. If for $m\ge 1$ $E'(G_{m-1})=0$ then we stop. If $E'(G_{m-1})\neq 0$ then we conduct the following three steps.

(1) Take any $\ff_m \in \D$ such that
\begin{equation}\label{7.3}
\<-E'(G_{m-1}),\ff_m\> \ge t_mE_\D(G_{m-1}).  
\end{equation}

(2) Choose $c_m>0$ from the equation
\begin{equation}\label{7.4}
 \mu(c_m) = \frac{t_mb}{2}c_mE_\D(G_{m-1})  
\end{equation}
provided it has a solution $c_m>0$ and set $c_m=1$ otherwise.

(3) Define
\begin{equation}\label{7.5}
G_m:=G_{m-1}+c_m\ff_m.  
\end{equation}

We note that equation (\ref{7.4}) is equivalent to the equation
$$
 \frac{\mu(c_m)}{c_m} = \frac{t_mb}{2}E_\D(G_{m-1}).
 $$
Our assumption $E'(G_{m-1})\neq 0$ implies that $E_\D(G_{m-1})> 0$. Therefore, the above equation either has a solution $c_m>0$ or $\mu(u)/u \le 
\frac{t_mb}{2}E_\D(G_{m-1})$ for all $u$. 

The greedy step (1) in the above algorithm is a standard greedy step which is based on $E'(G_{m-1})$. The choice of the coefficient $c_m$ from equation (\ref{7.4}) requires knowledge of both $E_\D(G_{m-1})$ and $\mu(u)$. The quantity $E_\D(G_{m-1})$ can be computed (in case $X$ is finite dimensional and $\D$ is finite). The function $\mu(u)$ comes from our assumption on $E$ and may be the one which does not describe smoothness of $E$ in the best way. Here is a modification of the GGA($\tau,b,\mu$) which does not use $\mu$.

{\bf Gradient E-Greedy Algorithm with parameters $(\tau)$ (GEGA$(\tau)$).}
Let $E$ be uniformly smooth convex function. For parameters $\tau=\{t_k\}$, $t_k\in(0,1]$   we define sequences
$\{G_m\}_{m=0}^\infty$, $\{\ff_m\}_{m=1}^\infty$, $\{c_m\}_{m=1}^\infty$ inductively. Let $G_0:=0$. If for $m\ge 1$ $E'(G_{m-1})=0$ then we stop. If $E'(G_{m-1})\neq 0$ then we conduct the following three steps.

(1) Take any $\ff_m \in \D$ such that
\begin{equation}\label{1.14}
\<-E'(G_{m-1}),\ff_m\> \ge t_mE_\D(G_{m-1}).  
\end{equation}

(2) Choose $c_m$ from the equation
\begin{equation}\label{1.15}
E(G_{m-1}+c_m\ff_m)=\min_cE(G_{m-1}+c\ff_m).  
\end{equation}

(3) Define
\begin{equation}\label{1.16}
G_m:=G_{m-1}+c_m\ff_m.  
\end{equation}

Our main interest in this paper is in analysis of greedy constructions of sparse approximants for convex optimization problems with respect to an arbitrary dictionary $\D$. We now give a comment that relates the above algorithms to classical gradient-type algorithms and thus justifies the use of the term {\it gradient} in the names of these algorithms. We specify our dictionary $\D$ to be the unit sphere $\mathcal S := \{g\in X:\|g\|=1\}$ of the space $X$. Then
$$
E_\D(x)= \|E'(x)\|_{X^*}.
$$
Therefore, the greedy step from the Gradient Based Expansion, the GGA($\tau,\C$), and the GGA($\tau,b,\mu$) takes the form: choose $\ff_m\in\D$ such that
$$
\<-E'(G_{m-1}),\ff_m\> \ge t_m\|E'(G_{m-1})\|_{X^*}.
$$
In particular, when $X=\R^n$ equipped with Euclidean norm and $t_m=1$ we obtain
$$
\ff_m = -E'(G_{m-1})/\|E'(G_{m-1})\|_2
$$
is a unit vector in the direction opposite to the gradient $E'(G_{m-1})$ direction. In this case the GGA($\{1\},b,\mu$) with $\mu(u)=\gamma u^2$ chooses the step size $c_m$ from the equation
$$
\gamma c_m^2 = \frac{b}{2}c_m\|E'(G_{m-1})\|_2\quad \Rightarrow\quad c_m =\frac{b}{2\gamma}\|E'(G_{m-1})\|_2.
$$
Thus
$$
G_m=G_{m-1} +c_m\ff_m = G_{m-1} - \frac{b}{2\gamma}E'(G_{m-1}),
$$
which describes the classical Gradient Method. 

\section{Convergence of the Gradient Based Expansion} 

In this section we 
assume that the sets
$$
D_C:=\{x:E(x)\le E(0)+C\}
$$
are bounded for all finite $C$ and that for any bounded set $\Omega$ we have
\begin{equation}\label{2.1}
\sup_{x\in \Omega}\|E'(x)\|_{X^*} <\infty.
\end{equation}

We begin with the following lemma

\begin{Lemma}\label{L2.1}  Let $E$ be Fr{\'e}chet differentiable convex function satisfying the above assumptions. Assume that the approximants $\{G_j\}_{j=0}^\infty$ and coefficients $\{c_j\}_{j=1}^\infty$ from the Gradient Based Expansion  satisfy the following two conditions
\begin{equation}\label{2.2}
\sum_{j=1}^\infty c_jE_\D(G_{j}) <\infty,  
\end{equation}
\begin{equation}\label{2.3}
\sum_{j=1}^\infty c_j =\infty.  
\end{equation}
Then 
\begin{equation}\label{2.4}
\liminf_{m\to \infty} E(G_m) =\inf_{x\in D} E(x).  
\end{equation}
\end{Lemma}
\begin{proof} By (\ref{1.3}) 
$$
E(G_m)-E(G_{m-1}) \le \<E'(G_m),G_m-G_{m-1} \> = c_m\<E'(G_m),\ff_m\>.
$$
This implies
$$
E(G_m)\le E(G_{m-1})+c_mE_\D(G_m).
$$
Using our assumption (\ref{2.2}) we obtain
$$
E(G_m)\le E(0) +\sum_{j=1}^mc_jE_\D(G_j) \le E(0)+C_1.
$$
By (\ref{2.1}) we get from here for all $m$
$$
\|E'(G_m)\|_{X^*}\le C_2.
$$
  Denote $s_n:=\sum_{j=1}^nc_j$. Then (\ref{2.3}) implies (see \cite{Bar}, p. 904) 
that
\begin{equation}\label{2.5}
\sum_{n=1}^\infty \frac{c_n}{s_n} =\infty.  
\end{equation}
Using (\ref{2.2})) we get
$$
\sum_{n=1}^\infty s_nE_\D(G_{n})\frac{c_n}{s_n} = \sum_{n=1}^\infty c_nE_\D(G_{n}) <\infty.
$$
Thus, by (\ref{2.5})
$$
\liminf_{n\to \infty} s_nE_\D(G_{n}) =0.
$$
 
Let
\begin{equation}\label{2.6}
\lim_{k\to \infty} s_{n_k}E_\D(G_{n_k}) =0.  
\end{equation}
Consider $\{E'(G_{n_k})\}$. A closed bounded set in the dual $X^*$ is weakly$^*$ compact (see \cite{HHZ}, p. 45). Let $\{F_i\}_{i=1}^\infty$, $F_i:= -E'(G_{n_{k_i}})$ be a $w^*$-convergent subsequence. Denote
$$
F:= w^*\text{-}\lim_{i\to\infty} F_i.
$$
We complete the proof of Lemma \ref{L2.1} by contradiction. We assume that (\ref{2.4}) does not hold, that is,  there exist $\a>0$ and $ N\in{\mathbb N}$ such that
\begin{equation}\label{2.7}
E(G_m)-\inf_{x\in D}E(x) \ge 2\a, \quad m \ge N,  
\end{equation}
and then derive a contradiction. 

We begin by deducing from (\ref{2.7}) that $F\neq 0$. Indeed, by (\ref{2.7}) there exists $f\in D$ such that 
\begin{equation}\label{2.8}
E(G_m)-E(f) \ge \a, \quad m \ge N.  
\end{equation}
By (\ref{1.3}) we obtain
\begin{equation}\label{2.9}
\<-E'(G_m),f-G_m\>\ge E(G_m)-E(f)\ge \a.
\end{equation}
Next, we have
\begin{equation}\label{2.10}
\<F,f\> =\lim_{i\to\infty} \<F_i,f\>  
\end{equation}
and
$$
|\<F_i,G_{n_{k_i}}\>| = |\<F_i, \sum_{j=1}^{n_{k_i}} c_j\ff_j\>| 
$$
\begin{equation}\label{2.11}
=  
|\sum_{j=1}^{n_{k_i}} c_j\<F_i,\ff_j\>| \le s_{n_{k_i}}E_\D(G_{n_{k_i}}) \to 0 
\end{equation}
for   $i\to \infty$. Relations (\ref{2.10}), (\ref{2.11}) and (\ref{2.9}) imply that $\<F,f\> \ge \a$, and hence $F\neq 0$.
 This implies that there exists $ g \in \D$ for which $\<F,g\>>0$. However,
$$
\<F,g\> =\lim_{i\to \infty} \<F_i,g\> \le \lim_{i\to \infty} E_\D(G_{n_{k_i}}) =0.  
$$
We have a contradiction, which completes the proof of Lemma \ref{L2.1}.
\end{proof}

\section{Convergence of GGA($\tau,\C$) and EGA($\C$)}

We begin with a simple lemma.
\begin{Lemma}\label{L3.1} Let  $f$,   $A>0$, be such that
$$
  f/A\in A_1(\D) .
$$
Then for
$$
G_k:=\sum_{j=1}^kc_j\ff_j,\quad \ff_j\in\D,\quad j=1,\dots,k,
$$
we have
$$
E_\D(G_k)\ge(E(G_k)-E(f))/(A+A_k),\quad A_k:=\sum_{j=1}^k|c_j|.
$$
\end{Lemma}
\begin{proof}  We have by (\ref{1.3})
\begin{equation}\label{3.1}
\<-E'(G_k),f-G_k\> \ge E(G_k)-E(f).
\end{equation}
Next,
\begin{equation}\label{3.2}
|\<-E'(G_k),f \>|\le AE_\D(G_k),
\end{equation}
\begin{equation}\label{3.3}
|\<-E'(G_k),G_k\>| \le E_\D(G_k)\sum_{j=1}^k |c_j|.
\end{equation}
Inequalities (\ref{3.1})--(\ref{3.3}) imply the statement of Lemma \ref{L3.1}.

\end{proof}

We now proceed to a convergence result for general uniformly smooth convex function $E$. 

\begin{Theorem}\label{T3.1} Let $E$ be a uniformly smooth convex function satisfying
\begin{equation}\label{3.4'}
 E(x+uy)-E(x)-u\<E'(x),y\>\le 2\mu(u),  
\end{equation}
for $ x\in D_2,$ $\|y\|=1,$ $ |u|\le 1$ with $\mu(u)=o(u)$ as $u\to 0$. Assume that the coefficients sequence $\C:=\{c_j\}$, $c_j\in[0,1]$ satisfies the conditions
\begin{equation}\label{3.4}
\sum_{k=1}^\infty c_k =\infty,  
\end{equation}
\begin{equation}\label{3.5}
\sum_{k=1}^\infty \mu(c_k)\le 1.  
\end{equation}
Then for the GGA($t,\C$) and for the EGA($\C$) we have for each dictionary $\D$  
$$
\lim_{m\to\infty}E(G_m)=\inf_{x\in D}E(x).
$$
\end{Theorem}
\begin{proof}   We give here a proof that works for both algorithms from Theorem \ref{T3.1}. Let $G_{m-1}$ be an approximate solution after $m-1$ iterations of either the GGA($t,\C$) or the EGA($\C$). Let $\ff_m$ be such that 
\begin{equation}\label{3.6'}
\<-E'(G_{m-1}),\ff_m\>\ge tE_\D(G_{m-1}).  
\end{equation}
Then
$$
\inf_{g\in\D}E(G_{m-1}+c_mg)\le E(G_{m-1}+c_m\ff_m).
$$
Thus, in both cases (GGA($t,\C$) and EGA($\C$)) it is sufficient to estimate $E(G_{m-1}+c_m\ff_m)$ with $\ff_m$ satisfying (\ref{3.6'}). By (\ref{3.4'}) under assumption that $G_{m-1}\in D_2$ we get
$$
E(G_{m-1}+c_m\ff_m)\le E(G_{m-1}) +c_m\<E'(G_{m-1}),\ff_m\>+2\mu(c_m).
$$
Using definition of $\ff_m$ we continue
\begin{equation}\label{3.6}
\le E(G_{m-1})-c_mtE_\D(G_{m-1})+2\mu(c_m).
\end{equation}
We now prove by induction that $G_m\in D_2$ for all $m$. Indeed, clearly $G_0\in D_2$.  Suppose that $G_k\in D_2$, $k=0,1,\dots, m-1$, then (\ref{3.6}) holds for all $k=1,\dots,m$ instead of $m$ and, therefore,
$$
E(G_m)\le E(0)+2\sum_{k=1}^m\mu(c_k) \le E(0)+2
$$
which implies that $G_m\in D_2$. 

Let  $f^\e$, $\e\ge0$, $A(\e)>0$, be such that
$$
E(f^\e)-b \le\e,\quad f^\e/A(\e)\in A_1(\D),\quad b:=\inf_{x\in D}E(x).
$$
Applying Lemma \ref{L3.1} we obtain from (\ref{3.6}) (with $A_k:=\sum_{j=1}^k c_j$)
\begin{equation}\label{3.7}
E(G_{m-1}+c_m\ff_m)\le E(G_{m-1})-\frac{tc_m(E(G_{m-1})-b-\e)}{A(\e)+A_{m-1}}+2\mu(c_m). 
\end{equation}
Denote
$$
a_n:=E(G_n)-b-\e.
$$
By (\ref{3.7}) we obtain
\begin{equation}\label{3.8}
a_m\le a_{m-1}(1-\theta_m)+2\mu(c_m).
\end{equation}
with
$$
\theta_m:= \frac{tc_m}{A(\e)+A_{m-1}}.
$$
We note that our assumption (\ref{3.4}) implies that
\begin{equation}\label{3.9}
\sum_{m=1}^\infty \theta_m =\infty.
\end{equation}
Without loss of generality we can assume that $A(\e)\ge 1$. Then $\theta_m\le 1$ and we get from (\ref{3.8}) 
\begin{equation}\label{3.10}
a_m\le a_{0}\prod_{j=1}^m(1-\theta_j)+2\mu(c_1)\prod_{j=2}^m(1-\theta_j)+\cdots + 2\mu(c_{m-1})(1-\theta_m) + 2\mu(c_m) .
\end{equation}
The properties (\ref{3.9}) and $\sum_m\mu(c_m) < \infty$ imply that
$$
\limsup_{m\to\infty}a_m \le 0.
$$
This completes the proof of Theorem \ref{T3.1}.
\end{proof}

 \section{Rate of convergence of GGA($\tau,\C$) and EGA($\C$)}

In this section we consider the GGA($t,\C$) and the EGA($\C$) with a specific sequence $\C$. For a special $\C$ we prove the rate of convergence results for the uniformly smooth convex functions with modulus of smoothness $\rho(E,u)\le \gamma u^q$, $q\in(1,2]$.

\begin{Theorem}\label{T4.1} Let $E$ be a uniformly smooth convex function with modulus of smoothness $\rho(E,u)\le \gamma u^q$, $q\in(1,2]$ on $D_2$. 
We set $s:=\frac{t+1}{t+q}$ and $\C_s:=\{ck^{-s}\}_{k=1}^\infty$ with $c$ chosen in such a way that $\gamma c^q \sum_{k=1}^\infty k^{-sq} \le 1$. Then the 
GGA($t,\C_s$) and EGA($\C_s$) (for this algorithm $t=1$) converge 
  with the following rate: for any $r\in(0,t(1-s))$
$$
E(G_m)-\inf_{\frac{x}{M}\in A_1(\D)}E(x) \le C(r,t,q,\gamma,M)m^{-r}.
$$
\end{Theorem}
\begin{proof} In the same way as in the proof of Theorem \ref{T3.1} we prove that $G_m\in D_2$ for all $m$. Then we use inequality (\ref{3.7}) proved in Section 3. Let  $f^\e$, $\e\ge0$, $M>0$, be such that
$$
E(f^\e)-b \le\e,\quad f^\e/M\in A_1(\D),\quad b:=\inf_{\frac{x}{M}\in A_1(\D)}E(x).
$$
Using the assumption $f^\e/M\in A_1(\D)$, we write (\ref{3.7}) with $A(\e)=M$
\begin{equation}\label{4.1}
E(G_{m-1}+c_m\ff_m)\le E(G_{m-1})-\frac{tc_m(E(G_{m-1})-b-\e)}{M+A_{m-1}}+2\gamma c_m^q. 
\end{equation}
We have
$$
A_{m-1}=c\sum_{k=1}^{m-1}k^{-s} \le c(1+\int_1^mx^{-s}dx)=c(1+(1-s)^{-1}(m^{1-s}-1)))
$$
 and
$$
M+A_{m-1}\le M+ c(1-s)^{-1}m^{1-s}.
$$
Therefore, for $m\ge N$ we have with $v:=(r+t(1-s))/2$
\begin{equation}\label{4.2}
\frac{tc_m}{M+A_{m-1}}\ge \frac{v+t(1-s)}{2m}.  
\end{equation}
We need the following   technical lemma. This lemma is a more general version of Lemma 2.1 from \cite{T1} (see also Remark 5.1 in \cite{T7} and Lemma 2.37 on p. 106 of \cite{Tbook}).
\begin{Lemma}\label{L4.1}  Let four positive numbers $\a < \beta \le 1$, $A$, $U\in \mathbb N$ be given and let a sequence $\{a_n\}_{n=1}^\infty$ have the following properties:  $ a_1<A$ and  we have for all $n\ge 2$
 \begin{equation}\label{4.3}
 a_n\le a_{n-1}+A(n-1)^{-\a};  
\end{equation}
 if for some $\nu \ge U$ we have
$$
a_\nu \ge A\nu^{-\a}
$$
then
\begin{equation}\label{4.4}
a_{\nu + 1} \le a_\nu (1- \beta/\nu). 
\end{equation}
Then there exists a constant $C=C(\a , \beta,A,U )$ such that for all $n=1,2,\dots $ we have
$$
a_n \le C n^{-\a} .
$$
 \end{Lemma}

We apply this lemma with $a_n:= E(G_n)-b-\e$, $\alpha:=r$, $\beta:=v:=(r+t(1-s))/2$, $U=N$ and $A$ specified later. Let us check the conditions (\ref{4.3}) and (\ref{4.4}) of Lemma \ref{L4.1}. By the inequality
$$
E(G_m)\le E(G_{m-1})+2\rho(E,c_m) \le E(G_{m-1})+2\gamma c^q m^{-sq}
$$
the condition (\ref{4.3}) holds for $A\ge 2\gamma c^q$. Assume that $a_m\ge Am^{-r}$. Then using $sq\ge 1+r$ we get
\begin{equation}\label{4.5}
c_m^q = c^q m^{-sq}\le c^q m^{-1-r}. 
\end{equation}
Setting $A$ to be big enough to satisfy
$$
2\gamma c_m^q \le \frac{A(t(1-s)-\beta)}{2m^{1+r}}
$$
 we obtain from (\ref{4.1}), (\ref{4.2}), and (\ref{4.5})
$$
a_{m+1}\le a_m(1-\beta/m)
$$
provided $a_m\ge Am^{-r}$. Thus (\ref{4.4}) holds. Applying Lemma \ref{L4.1} we get
$$
a_m\le C(r,t,q,\gamma,M)m^{-r}. 
$$
\end{proof}

We note that in the special case when $\D$ is the unit sphere $\mathcal S$ of $X$ the rate of convergence in Theorem \ref{T4.1} can be improved.

\begin{Theorem}\label{T4.2} Let $E$ be a uniformly smooth convex function with modulus of smoothness $\rho(E,u)\le \gamma u^q$, $q\in(1,2]$ on $D_2$ which we assume to be bounded. 
For a $\delta\in(0,1)$ we set $s:=1-\delta$  and $\C_s:=\{ck^{-s}\}_{k=1}^\infty$ with $c:=c(\delta)$ chosen in such a way that $\gamma c^q \sum_{k=1}^\infty k^{-sq} \le 1$. Suppose $\D=\mathcal S$. Then the 
GGA($t,\C_s$) and EGA($\C_s$) (for this algorithm $t=1$) converge 
  with the following rate:  
$$
E(G_m)-\inf_{x\in D}E(x) \le C(E,\delta,q,\gamma,t)m^{-s(q-1)}.
$$
\end{Theorem}
\begin{proof} As we already mentioned in the Introduction in the case $\D=\mathcal S$ we have
$$
E_\D(x)= \|E'(x)\|_{X^*}.
$$
By (\ref{3.4'}) under assumption that $G_{m-1}\in D_2$ we get
$$
E(G_{m-1}+c_m\ff_m)\le E(G_{m-1}) +c_m\<E'(G_{m-1}),\ff_m\>+2\gamma(c_m)^q.
$$
Using definition of $\ff_m$ we continue
\begin{equation}\label{3.6E}
\le E(G_{m-1})-c_mt\|E'(G_{m-1})\|_{X^*}+2\gamma(c_m)^q.
\end{equation}
As in the proof of Theorem \ref{T3.1} we derive from here that $G_m\in D_2$ for all $m$. Using notation  $a_m:=E(G_m)-\inf_{x\in D}E(x)$ we obtain 
\begin{equation}\label{5.24E}
a_{m-1}  = \sup_{f\in D}( E(G_{m-1})-E(f)) \le \|E'(G_{m-1})\|_{X^*}\sup_{f\in D}\|G_{m-1}-f\|.
\end{equation}
Inequality (\ref{5.24E}) and our assumption that $D_2$ is bounded imply
$$
\|E'(G_{m-1})\|_{X^*} \ge a_{m-1}/C_1.
$$
Substituting this bound into (\ref{3.6E}) we get
\begin{equation}\label{4.9}
a_m\le a_{m-1}\left(1-tc_m C_1^{-1}\right)+2\gamma(c_m)^q.
\end{equation}
As in the proof of Theorem \ref{T4.1} we use Lemma \ref{L4.1}. It is clear that 
for $a_{m-1}$ satisfying 
$$
a_{m-1}\ge Am^{-s(q-1)}
$$
with large enough $A$ we have
$$
a_{m-1}tc_m C_1^{-1}\ge 4\gamma(c_m)^q.
$$
Therefore, (\ref{4.9}) gives in this case
\begin{equation}\label{4.10}
a_m\le a_{m-1}\left(1-\frac{tc_m}{2 C_1}\right) .
\end{equation}
It follows from the definition of $c_m$ that 
$$
\frac{tc_m}{2 C_1} \ge \frac{s(q-1)+1}{m-1}\quad\text{for}\quad m\ge U.
$$
Thus by Lemma \ref{L4.1} we obtain
$$
a_m\le C(E,\delta,q,\gamma,t)m^{-s(q-1)}
$$
which proves Theorem \ref{T4.2}.
\end{proof}

\section{Convergence and rate of convergence of the GGA($\tau,b,\mu$)}

We begin with a convergence result.
\begin{Theorem}\label{T5.1} Let $E$ be a uniformly smooth convex function with the modulus of smoothness $\rho(E,D,u)$ and let $\mu(u)$ be a continuous majorant of $\rho(E,D,u)$ with the property $\mu(u)/u \downarrow 0$ as $u\to +0$. Assume that for $x\in D$
$$
\|E'(x)\|_{X^*} \le C_D.
$$
Then, for any $t\in (0,1]$ and $b\in (0,1)$ we have for the GGA($t,b,\mu$)
\begin{equation}\label{5.1}
\lim_{m\to\infty}E(G_{m})=\inf_{x\in D}E(x).
\end{equation}
\end{Theorem}
\begin{proof} In this case $\tau =\{t\}$, $t\in (0,1]$. We have by (\ref{3.4'}) under assumption that $G_{m-1}\in D$ 
$$
E(G_{m-1}+c_m\ff_m)\le E(G_{m-1}) +c_m\<E'(G_{m-1}),\ff_m\>+2\mu(c_m).
$$
Using definition of $\ff_m$ we continue
\begin{equation}\label{7.24}
\le E(G_{m-1})-c_mtE_\D(G_{m-1})+2\mu(c_m).
\end{equation}
Using the choice of $c_m$ we find
\begin{equation}\label{7.25}
E(G_m) \le E(G_{m-1})-t(1-b)c_mE_\D(G_{m-1}).  
\end{equation}
In particular, (\ref{7.25}) implies that $\{E(G_m)\}$ is a monotone decreasing sequence and therefore our assumption that $G_{m-1}\in D$ implies that $G_{m}\in D$. Clearly, $G_0\in D$. Thus we obtain that $G_{m}\in D$ for all $m$. Also, (\ref{7.25}) implies that
$$
t(1-b)c_mE_\D(G_{m-1}) \le E(G_{m-1}) -E(G_m).
$$
Thus
\begin{equation}\label{7.26}
\sum_{m=1}^\infty c_mE_\D(G_{m-1}) <\infty.  
\end{equation}
We have the following two cases:
$$
(I)\quad \sum_{m=1}^\infty c_m =\infty, \qquad (II)\quad \sum_{m=1}^\infty c_m <\infty.
$$
First, we consider  case (I). Our argument here is as in Lemma \ref{2.1}. 
  Denote $s_n:=\sum_{j=1}^nc_j$. Then our assumption implies (see \cite{Bar}, p. 904) 
that
\begin{equation}\label{2.5g}
\sum_{n=1}^\infty \frac{c_n}{s_n} =\infty.  
\end{equation}
Using (\ref{7.26})) we get
$$
\sum_{n=1}^\infty s_nE_\D(G_{n-1})\frac{c_n}{s_n} = \sum_{n=1}^\infty c_nE_\D(G_{n-1}) <\infty.
$$
Thus, by (\ref{2.5g})
$$
\liminf_{n\to \infty} s_nE_\D(G_{n-1}) =0.
$$
Clearly, the above relation implies 
$$
\liminf_{n\to \infty} s_nE_\D(G_{n}) =0.
$$
The rest of the proof in this case repeats the corresponding part from the proof of Lemma \ref{L2.1}. As a result we obtain
$$
\liminf_{m\to 0}E(G_m)=\inf_{x\in D}E(x).
$$
Monotonicity of $\{E(G_m)\}$ implies that we can replace liminf by lim in the above relation.
 
Second, we consider the case (II). Our assumption implies that $c_m\to 0$ as $m\to \infty$. From the definition (\ref{7.4}) of $c_m$ we obtain
\begin{equation}\label{5.5}
   E_\D(G_{m-1}) = \frac{2}{tb} \mu(c_m)/c_m \to 0, \quad m\to \infty.
\end{equation}
We show that relation (\ref{5.5}) implies the following two properties (\ref{5.6}) and (\ref{5.7})
\begin{equation}\label{5.6}
\lim_{m\to 0}\<E'(G_m),G_m\> = 0,
\end{equation}
\begin{equation}\label{5.7}
\lim_{m\to 0}\<E'(G_m),f\> = 0. 
\end{equation}
Indeed, for (\ref{5.6}) we have
$$
|\<E'(G_m),G_m\>| = |\sum_{j=1}^m \<E'(G_m),\ff_j\>c_j | \le E_\D(G_m)\sum_{j=1}^mc_j \to 0.
$$
We now prove (\ref{5.7}). For arbitrary $\e>0$ find $f^\e$ such that 
$$
\|f-f^\e\|\le \e,\quad f^\e/A(\e)\in A_1(\D),
$$
with some $A(\e)$. 
Then
$$
|\<E'(G_m),f\>|=|\<E'(G_m),f^\e\>+\<E'(G_m),f-f^\e\>|\le E_\D(G_m)A(\e)+C_D\e.
$$
We complete the proof of case (II) by contradiction. We assume that (\ref{5.1}) does not hold, that is,  there exist $\a>0$ and $ N\in{\mathbb N}$ such that
\begin{equation}\label{2.7g}
E(G_m)-\inf_{x\in D}E(x) \ge 2\a, \quad m \ge N,  
\end{equation}
and then derive a contradiction. 
By (\ref{2.7g}) there exists $f\in D$ such that 
\begin{equation}\label{2.8g}
E(G_m)-E(f) \ge \a, \quad m \ge N.  
\end{equation}
By (\ref{1.3}) we obtain
\begin{equation}\label{2.9g}
\<-E'(G_m),f-G_m\>\ge E(G_m)-E(f)\ge \a.
\end{equation}
This contradicts to (\ref{5.6}) and (\ref{5.7}).

\end{proof}

\begin{Theorem}\label{T5.1E} Let $E$ be a uniformly smooth convex function.  Assume that for $x\in D$
$$
\|E'(x)\|_{X^*} \le C_D.
$$
Then, for any $t\in (0,1]$   we have for the GEGA($\{t\}$)
\begin{equation}\label{5.1E}
\lim_{m\to\infty}E(G_{m})=\inf_{x\in D}E(x).
\end{equation}
\end{Theorem}
\begin{proof} Let $E$ be a uniformly smooth convex function with the modulus of smoothness $\rho(E,D,u)$ and let $\mu(u)$ be a continuous majorant of $\rho(E,D,u)$ with the property $\mu(u)/u \downarrow 0$ as $u\to +0$. As in (\ref{7.25}) we obtain
\begin{equation}\label{7.25E}
E(G_m) \le E(G_{m-1}+c_m'\ff_m) \le E(G_{m-1})-t(1-b)c_m'E_\D(G_{m-1}) 
\end{equation}
with $c_m'$ chosen from the equation
\begin{equation}\label{7.4E}
 \mu(c_m') = \frac{tb}{2}c_m'E_\D(G_{m-1})  
\end{equation}
with some fixed $b\in(0,1)$. 

The proof of Theorem \ref{T5.1} used only assumptions on $E$ and analogs of relations  (\ref{7.25E}) and (\ref{7.4E}). Therefore the same proof gives 
(\ref{5.1E}).
\end{proof}

We proceed to study  the rate of convergence of the GGA$(\tau,b,\mu)$ for the uniformly smooth convex function with the power-type majorant of modulus of smoothness: $\rho(E,D,u)\le \mu(u)= \gamma u^q$, $1<q\le 2$.  
\begin{Theorem}\label{T5.2} Let $\tau :=\{t_k\}_{k=1}^\infty$ be a nonincreasing sequence $1\ge t_1\ge t_2 \dots >0$ and $b\in (0,1)$. Assume that uniformly smooth convex function $E$ has a modulus of smoothness $\rho(E,D,u)\le \gamma u^q$, $q\in (1,2]$. Denote $\mu(u) = \gamma u^q$. Then  the rate of convergence of the GGA$(\tau,b,\mu)$ is given by 
$$
E(G_m)-\inf_{x\in A_1(\D)}E(x)\le C(b,\gamma,q)(1+\sum_{k=1}^mt_k^p)^{-\frac{t_m(1-b)(q-1)}{q+t_m(1-b)}}, \quad p:= \frac{q}{q-1}.
$$
\end{Theorem}
\begin{proof} As in (\ref{7.25}), we get
\begin{equation}\label{7.30}
E(G_m)\le E(G_{m-1}) -t_m(1-b)c_mE_\D(G_{m-1}).  
\end{equation}
Thus we need to estimate  $c_mE_\D(G_{m-1})$ from below. 
Denote $b_n:= 1+\sum_{j=1}^nc_j$.  By Lemma \ref{3.1}   we obtain
\begin{equation}\label{7.32}
E_\D(G_{m-1})   \ge (E(G_{m-1}) - w)/b_{m-1},\quad w:=\inf_{x\in A_1(\D)}E(x) .  
\end{equation}
Substituting (\ref{7.32}) into (\ref{7.30}) and using notation $a_m:=E(G_m)-w$ we get
\begin{equation}\label{7.33}
a_m \le a_{m-1}(1-t_m(1-b)c_m/b_{m-1}).  
\end{equation}
From the definition of $b_m$ we find
$$
b_m = b_{m-1} +c_m = b_{m-1}(1+c_m/b_{m-1}).
$$
Using the inequality
$$
(1+x)^\a \le 1+\a x, \quad 0\le \a\le 1, \quad x\ge 0,
$$
we obtain
\begin{equation}\label{7.34}
b_m^{t_m(1-b)}  \le  b_{m-1}^{t_m(1-b)}(1+t_m(1-b)c_m/b_{m-1}).  
\end{equation}
Multiplying (\ref{7.33}) and (\ref{7.34}), and using that $t_m\le t_{m-1}$, we get
\begin{equation}\label{7.35}
a_mb_m^{t_m(1-b)}  \le a_{m-1} b_{m-1}^{t_{m-1}(1-b)} \le a_0.  
\end{equation}
The function $\mu(u)/u = \gamma u^{q-1}$ is increasing on $[0,\infty)$. Therefore the $c_m$ is greater than or equal to $c_m'$ from (see (\ref{7.32}))
\begin{equation}\label{7.36}
\ga  (c_m')^q = \frac{t_mb}{2}c_m'a_{m-1}/b_{m-1},  
\end{equation}
\begin{equation}\label{7.37}
c_m' = \left(\frac{t_mb}{2\ga}\right)^{\frac{1}{q-1}}\left(\frac{a_{m-1}}{b_{m-1}}\right)^{\frac{1}{q-1}}.  
\end{equation}
Using notations
$$
p:=\frac{q}{q-1},\qquad A^{-1} := (1-b)\left(\frac{b}{2\ga}\right)^{\frac{1}{q-1}},
$$
we obtain  
\begin{equation}\label{7.38}
a_m\le a_{m-1} \left(1-\frac{t_m^p}{A}\frac{a_{m-1}^{\frac{1}{q-1}}}{b_{m-1}^p}\right),  
\end{equation}
from (\ref{7.33}) and (\ref{7.37}). Noting that $b_m\ge b_{m-1}$, we infer from (\ref{7.38}) that
\begin{equation}\label{7.39}
\frac{a_m^{\frac{1}{q-1}}}{b_m^p} \le \frac{a_{m-1}^{\frac{1}{q-1}}}{b_{m-1}^p}\left(1-\frac{t_m^p}{A}\frac{a_{m-1}^{\frac{1}{q-1}}}{b_{m-1}^p}\right).  
\end{equation}
We obtain from (\ref{7.39}) by an analog of Lemma 2.16 from Chapter 2  of \cite{Tbook} (see \cite{T13}, Lemma 3.1)
\begin{equation}\label{7.40}
\frac{a_m^{\frac{1}{q-1}}}{b_m^p} \le C(E,b,\gamma)\left(1+\sum_{k=1}^m t_k^p\right)^{-1}.  
\end{equation}
Combining (\ref{7.35}) and (\ref{7.40}), we get 
$$
a_m\le C(E,b,\gamma,q)(1+\sum_{k=1}^m t_k^p)^{-\frac{t_m(1-b)(q-1)}{q+t_m(1-b)}}, \quad p:= \frac{q}{q-1}.
$$
This completes the proof of Theorem \ref{T5.2}. 
\end{proof} 

We note that in the special case when $\D$ is the unit sphere $\mathcal S$ of $X$ the rate of convergence in Theorem \ref{T5.2} can be improved.

\begin{Theorem}\label{T5.3} Let $\tau :=\{t_k\}_{k=1}^\infty$ be a weakness sequence $t_k\in[0,1]$ and $b\in (0,1)$. Assume that uniformly smooth convex function $E$ has a modulus of smoothness $\rho(E,D,u)\le \gamma u^q$, $q\in (1,2]$. Denote $\mu(u) = \gamma u^q$. Suppose that $\D=\mathcal S$. Then  the rate of convergence of the GGA$(\tau,b,\mu)$ is given by 
\begin{equation}\label{5.22}
E(G_m)-\inf_{x\in  D}E(x)\le C(E,b,\gamma,q)\left(1+\sum_{k=1}^mt_k^p\right)^{1-q} , \quad p:= \frac{q}{q-1}.
\end{equation}
\end{Theorem} 
\begin{proof} As we already mentioned in the Introduction in the case $\D=\mathcal S$ we have
$$
E_\D(x)= \|E'(x)\|_{X^*}.
$$
For $\mu(u) = \gamma u^q$ we obtain for the $c_m$
$$
\gamma c_m^q = \frac{t_mb}{2}c_m\|E'(G_{m-1})\|_{X^*}\quad \Rightarrow\quad c_m =\left(\frac{t_mb}{2\gamma}\|E'(G_{m-1})\|_{X^*}\right)^{\frac{1}{q-1}}.
$$
Therefore, by (\ref{7.30}) we get
\begin{equation}\label{5.23}
E(G_m)\le E(G_{m-1}) -t_m(1-b)\left(\frac{t_mb}{2\gamma}\|E'(G_{m-1})\|_{X^*}\right)^{\frac{1}{q-1}}\|E'(G_{m-1})\|_{X^*}.
\end{equation}
Equation (\ref{5.23}) implies that $G_m\in D$ for all $m$.
Using the notation $a_m:=E(G_m)-w$, $w:=\inf_{x\in D}E(x)$ we obtain 
\begin{equation}\label{5.24}
a_{m-1} =  \sup_{f\in D}( E(G_{m-1})-E(f)) \le \|E'(G_{m-1})\|_{X^*}\sup_{f\in D}\|G_{m-1}-f\|.
\end{equation}
Inequality (\ref{5.24}) and our assumption that $D$ is bounded imply
$$
\|E'(G_{m-1})\|_{X^*} \ge a_{m-1}/C_1.
$$
Substituting this bound into (\ref{5.23}) we get
\begin{equation}\label{5.25}
a_m\le a_{m-1}\left(1-t_m^pa_{m-1}^{\frac{1}{q-1}}C_2^{-1}\right).
\end{equation}
Inequality (\ref{5.25}) is similar to (\ref{7.38}). We derive (\ref{5.22}) from (\ref{5.25}) in the same way as (\ref{7.40}) was derived from (\ref{7.38}).
\end{proof}

{\bf Acknowledgements.} This paper was motivated by the IMA Annual Program Workshop "Machine Learning: Theory and Computation" (March 26--30, 2012), in particular, by talks of Steve Wright and Pradeep Ravikumar. 
The author is very thankful to Arkadi Nemirovski for an interesting discussion of the results and for his remarks.

\end{document}